\documentclass[a4paper,UKenglish,cleveref, autoref, thm-restate]{tgdk-v2021}

\title{RDFdL: Integrating RDF with Differential Dynamic Logic}

\author{Yuyang Li}
{Institute AIFB, Karlsruhe Institute of Technology (KIT), Karlsruhe, Germany}
{Yuyang.Li@kit.edu}
{https://orcid.org/0009-0001-4924-7720}
{}

\author{Lukas Kubelka}
{Institute AIFB, Karlsruhe Institute of Technology (KIT), Karlsruhe, Germany}
{lukas.kubelka@kit.edu}
{https://orcid.org/0009-0008-1828-5207}
{}

\author{Julia Butte}
{Institute AIFB, Karlsruhe Institute of Technology (KIT), Karlsruhe, Germany}
{julia.butte@kit.edu}
{https://orcid.org/0009-0003-5066-8412}
{}

\author{Tobias K{\"a}fer}
{Institute AIFB, Karlsruhe Institute of Technology (KIT), Karlsruhe, Germany}
{tobias.kaefer@kit.edu}
{https://orcid.org/0000-0003-0576-7457}
{}

\authorrunning{Y. Li et al.} 

\Copyright{Yuyang Li and Tobias K{\"a}fer} 



\category{} 

\relatedversion{} 

\nolinenumbers

\Volume{1}
\Issue{1}
\Article{42}
\DateSubmission{Date of submission}
\DateAcceptance{Date of acceptance}
\DatePublished{Date of publishing}
\SectionAreaEditor{TGDK section area editor}

\usepackage{booktabs}
\usepackage{float}
\usepackage{times}
\usepackage{soul}
\usepackage{url}
\usepackage[utf8]{inputenc}
\usepackage{graphicx}
\usepackage{amsmath}
\usepackage{amsthm}
\usepackage{amssymb}
\usepackage{stmaryrd}
\newcommand{\R}{\mathbb{R}}
\newcommand{\mode}{\mathsf{mode}}
\newcommand{\isNext}{\ensuremath{\mathsf{isNext}}}
\newcommand{\modeChange}{\ensuremath{\mathsf{modeChange}}}
\newcommand{\Inv}{\mathsf{Inv}}
\usepackage{booktabs}
\usepackage{algorithm}
\usepackage{algorithmic}
\usepackage{tabularx}
\usepackage{array}
\newcolumntype{Y}{>{\centering\arraybackslash}X}
\newfloat{listing}{tbp}{lop}
\floatname{listing}{Listing}
\hypersetup{hidelinks}
\begin{document}

\maketitle

\begin{abstract}
Knowledge graphs modeled in RDF are powerful for describing static knowledge, but they cannot capture or reason about the dynamic behavior of physical systems, e.g., systems described by differential equations, which is a critical gap for AI-driven cyber-physical systems. To solve this, we propose RDFdL, a framework that integrates RDF with Differential Dynamic Logic (dL) to represent and reason about both static knowledge and the continuous dynamics of physical systems. For the dynamic part, we syntactically represent differential equations and ranges in the state space in RDF and SHACL, and provide semantics using a translation to dL. Linking RDF and dL through their shared foundation in first-order logic achieves a unique integration: verification results for safety and reachability properties in the dynamic logic domain become available as entailment to SPARQL queries over RDF data. We implement the pipeline using Apache Jena for ontology-driven RDF reasoning and KeYmaera X, the theorem prover for dL, and sketch its applicability in manufacturing.

\end{abstract}`

\section{Introduction}
\label{sec:introduction}

Knowledge graphs are used as a data layer for cyber-physical and industrial systems. They provide globally identifiable resources, typed relations, graph schemas, provenance, and query mechanisms for information such as devices, sensors, controllers, products, operators, maintenance responsibilities, and process configurations. In manufacturing settings, such information is often represented through digital-twin descriptions and asset models, and RDF is a natural representation language because it supports open graph integration across organizational and system boundaries. These graph models are well-suited for answering structural questions: which device belongs to which production cell, which button controls which actuator, which technician is responsible for which machine, or which product is processed by which unit.

However, many questions about cyber-physical systems are not only structural. They also depend on how the physical process evolves over time. A heating device changes temperature according to a differential equation; a tank level changes according to inflow and outflow; a boiler evolves according to coupled pressure and water-level dynamics. Whether a transition from one operating situation to another is possible, safe, or forbidden cannot be decided from RDF triples alone. Standard RDF, RDFS, OWL, SHACL, and SPARQL provide graph representation, schema validation, and graph querying, but by themselves, they do not solve reachability or safety problems for systems whose behavior is governed by ordinary differential equations (ODEs). A graph may state that an oven has an \texttt{On} mode and an \texttt{Off} mode, but this does not establish whether switching modes at a particular temperature preserves a safety bound.

Hybrid-system verification addresses precisely this kind of continuous-discrete behavior. Differential dynamic logic (dL) provides a specification and proof language for hybrid programs, in which discrete control actions and continuous ODE evolution can be combined into a single formal model~\cite{Platzer18,10.1007/978-3-319-21401-6_36}. In dL, one can state and prove properties such as "every execution of a controller preserves a temperature bound" or "there exists a continuous evolution from one temperature region to another under a given ODE and evolution domain." Theorem provers such as KeYmaera X provide deductive support for such proofs. Yet dL models are not knowledge graphs. They do not, by themselves, provide RDF identifiers, graph schema validation, SPARQL path queries, links to device metadata, provenance, maintenance information, product descriptions, or integration with existing digital twin data.

The central problem is therefore neither to replace hybrid-system theorem provers with RDF reasoning nor to replace RDF knowledge graphs with dL models, but to build a sound bridge between the two: verified continuous and discrete behavior should become queryable as graph data, while graph data should provide the symbolic state regions, device modes, metadata, and candidate transitions from which proof obligations are generated. This bridge is useful whenever users need to ask graph queries whose answers depend on both static metadata and formally verified dynamic behavior.

To bridge this gap, we integrate Differential Dynamic Logic (dL) with RDF. dL is a formal specification and verification language for hybrid systems that can model continuous dynamics and prove system properties \cite{Platzer18}. In dL, we can describe a process with discrete transitions, such as controller decisions or mode switches, alongside continuous evolution governed by Ordinary Differential Equations (ODEs) for the physical process \cite{Platzer18}. dL allows specifying safety properties, such as invariants, bounds, and avoidance of undesired states, and formally verifying that these properties hold for all possible behaviors of the hybrid system. Tools such as KeYmaera X implement dL using a theorem-proving approach, enabling verification of correctness properties of hybrid programs \cite{10.1007/978-3-319-21401-6_36}. For example, "if the system starts safely, it will never violate the safety condition." This deductive verification has been successfully applied to complex hybrid systems. It has been used to verify industrial-scale cyber-physical systems, such as railway control systems, air traffic collision avoidance systems, and autonomous driving controllers \cite{Platzer18}.

Consider a production line in which an oven heats a product. An operations engineer asks: \textit{"Along the intended heat-up path of the process, at which point does the first mode switch occur, which button triggers it, and which technician is responsible for the device?"} The first half of this question is about \textit{verified dynamic behavior} whether the intended transitions are actually realizable under the oven's heating ODE. The second half is about \textit{static metadata} such as buttons, devices, and responsibilities that live naturally in an RDF graph. No existing tool answers both halves of the query in one query.

Thus, our contribution is to bridge the gap between RDF and formal verification for hybrid systems. We propose a novel approach that starts with an RDF-based knowledge model of a cyber-physical process, automatically derives a formal model suitable for verification in dL, and integrates dL reasoning results into RDF reasoning and querying. In summary, our contributions are as follows:

\begin{itemize}
    \item \textbf{RDF-based representation of hybrid systems:} We introduce a formal representation of hybrid systems in RDF, capturing continuous dynamics (ODEs) and SHACL-based state definition in RDF.
    \item \textbf{dL formula generation and verification:} We present the semantics of this representation by providing a translation from the RDF/SHACL hybrid systems representation to differential dynamic logic (dL).
    \item \textbf{RDFdL framework:} We provide the RDFdL framework using Apache Jena for reasoning and querying, and KeYmaera X for dL proving.
\end{itemize}

The remainder of the paper is organized as follows. Section~\ref{sec:motivating-example} introduces the motivating example and problem setting in detail. Section~\ref{sec:preliminaries} recalls the required background on RDF, SHACL, SPARQL, hybrid systems, and dL. Section~\ref{section:formal_method} establishes formal guarantees and discusses expressiveness and complexity. Section~\ref{sec:approach} describes the implementation. Section~\ref{sec:Evaluation} reports the evaluation. Section~\ref{sec:Related_work} discusses related work, and Section~\ref{sec:Conclusion} concludes.

\section{Preliminaries}
\label{sec:preliminaries}

This section introduces the background needed in the rest of the paper. We first recall the role of RDF, SHACL, and SPARQL in RDFdL. We then introduce the fragment of hybrid systems and differential dynamic logic (dL) used in this paper. Finally, we clarify the distinction between concrete dL states and RDFdL state regions, because this distinction is central to our graph-based representation.

\subsection{RDF, SHACL, and SPARQL}
We briefly introduce the Semantic Web technologies we apply: RDF, SHACL, SPARQL, and their roles in our approach.
We represent static and dynamic information about the hybrid system as RDF data, use SHACL to define state definitions, and SPARQL to query the data, while considering entailed data about state transitions from dL.

\paragraph{RDF (Resource Description Framework).} RDF\footnote{\url{https://www.w3.org/TR/rdf11-concepts/}} is a W3C Recommendation for a data model that encodes binary predicates in a labeled directed graph. An RDF graph is defined as a set of \emph{triples}, each triple in the form $\langle subject,\ predicate,\ object\rangle$.
In RDF, objects can be URIs\footnote{\url{http://www.ietf.org/rfc/rfc3986} abbreviated as CURIEs (see \url{http://www.w3.org/TR/curie}) according to the practices encoded in \url{https://prefix.cc/}. On top, the empty prefix \texttt{:} denotes our anonymized vocabulary "\url{http://anonymous.example/rdfdl/vocabulary\#}", and \texttt{ex:} examples "\url{http://example.org/\#}".} global resource identifiers, blank nodes (document-scoped resource identifiers), or literals (data values); subjects can be URIs or blank nodes; predicates can only be URIs.
In the paper, we use the Turtle syntax\footnote{\url{https://www.w3.org/TR/turtle/}} for RDF.

 In RDFdL, RDF is used to represent devices, variables, modes, symbolic state regions, ODE records, candidate transitions, verified transitions, proof artifacts, and ordinary metadata.

\paragraph{SHACL (Shapes Constraint Language).} SHACL\footnote{\url{https://www.w3.org/TR/shacl/}} is a W3C Recommendation for a constraint language for RDF data. It allows defining \emph{shapes}, i.\,e.\ conditions that can be used to validate RDF graphs. SHACL shapes can be encoded in RDF. In our work, we use SHACL to express cardinality constraints on a node's incoming/outgoing edges and value range constraints on edges' ends. We use SHACL to encode the states we define and their value ranges.

In RDFdL, SHACL is used to validate real-time input state. It validates observations and well-formed RDF descriptions. It does not define the semantics of ODEs, continuous evolution, or dL formulas. The semantics of hybrid-system behavior are given by the dL proof obligations introduced below. For example, SHACL may validate that an observation with temperature \(185\) and mode \(\mathsf{Off}\) conforms to the shape associated with a high-temperature off-region. This validation identifies the symbolic region to which the observation belongs. It does not prove that the system can safely reach that region. dL proof obligations handle reachability and safety.

\paragraph{SPARQL (SPARQL Protocol and Query Language).} In the SPARQL\footnote{\url{https://www.w3.org/TR/sparql11-query/}} query language for RDF, a W3C recommendation, we can define Basic Graph Patterns (BGPs) that need to be matched on RDF graphs to extract specific information. BGPs consist of triple patterns, i.\,e.\ triples where in each position, a variable may appear. FILTERs can be applied to the thus-gained results. On top, SPARQL offers so-called property paths, using which arbitrary sequences of properties can be queried. We use SPARQL to formulate queries about both the static and the dynamic parts. Property paths are especially useful to formulate reachability queries on sequences of states.

In our approach, we can also use SPARQL to fetch basic graph patterns to static facts (device types, wiring, and mode records). and to traverse materialized \texttt{:next} and \texttt{:modeChange} edges, thereby answering reachability queries such as “\emph{find happiness path from raw to finished} and querying the verified transition graph after dL proof results have been materialized.

\paragraph{Entailment in SPARQL and Apache Jena}
SPARQL can match BGPs on RDF graphs to derive SPARQL query results. To also consider results that can be derived by applying formal reasoning on the RDF graphs, multiple options are applied in practice: The SPARQL Entailment W3C Recommendation\footnote{\url{https://www.w3.org/TR/sparql11-entailment/}} defines how BGP matching can be extended to consider entailment regimes such as RDFS and OWL. However, the evaluation of property paths is not defined based on BGP matching; thus, the reasoning results are not available in property paths of arbitrary length~\cite{Glimm:13:SER}—and we require such property paths for reachability queries.
Reasoning in Apache Jena, however, is implemented based on a RETE engine or a tabled Datalog engine or a combination of both in a way that queries are executed over a combination of the original data and the inferred data, where both sources of data are treated the same~\cite{jena-inference}, such that property path queries with arbitrary length can indeed be evaluated also over inferred triples.

\subsection{Differential Dynamic Logic}

We use differential dynamic logic (dL) in the paper of Platzer~\cite{Platzer18}. Differential dynamic logic is a logic for specifying and verifying hybrid systems, i.e., systems whose behavior combines discrete transitions with continuous evolution along differential equations. In dL, system behavior is represented by \emph{hybrid programs}, and properties of such programs are stated as logical formulas.

Let \(\mathcal{V}\) be a finite set of variables. A state \(\nu\) is a valuation that assigns a value to each variable in \(\mathcal{V}\). For the continuous variables used in this paper, the values are real numbers. We write \(\nu(x)\) for the value of the variable \(x\) in state \(\nu\).

The fragment of hybrid programs used in this paper is generated by the following grammar:
\[
\alpha,\beta ::= x := \theta
\mid ?\phi
\mid x_1'=\theta_1,\ldots,x_n'=\theta_n \ \&\ Q
\mid \alpha;\beta
\mid \alpha \cup \beta
\mid \alpha^\ast .
\]
Here, \(x := \theta\) is a discrete assignment, \(?\phi\) is a test that may proceed only if \(\phi\) holds, \(x_1'=\theta_1,\ldots,x_n'=\theta_n \ \&\ Q\) is continuous evolution along an ODE system within the evolution domain \(Q\), \(\alpha;\beta\) is sequential composition, \(\alpha \cup \beta\) is nondeterministic choice, and \(\alpha^\ast\) is finite repetition.

The semantics of a hybrid program \(\alpha\) is a binary transition relation
\[
\llbracket \alpha \rrbracket \subseteq \mathcal{S} \times \mathcal{S},
\]
where \(\mathcal{S}\) is the set of states. If
\[
(\nu,\omega) \in \llbracket \alpha \rrbracket,
\]
then there is an execution of \(\alpha\) that starts in state \(\nu\) and terminates in state \(\omega\). Thus, in this paper, an execution of a hybrid program is understood semantically as a transition between an initial state and a final state in the relation denoted by the program.

The dL formulas used in this paper are generated by:
\[
\phi,\psi ::= \theta_1 \sim \theta_2
\mid \neg \phi
\mid \phi \wedge \psi
\mid \phi \vee \psi
\mid \phi \rightarrow \psi
\mid \forall x\,\phi
\mid \exists x\,\phi
\mid [\alpha]\phi
\mid \langle \alpha\rangle\phi ,
\]
where \(\sim \in \{=,\neq,<,\leq,>,\geq\}\), \(\theta_1,\theta_2\) are arithmetic terms, and \(\alpha\) is a hybrid program.

The modal formulas are interpreted over the transition relation of the hybrid program:
\[
\nu \models [\alpha]\phi
\quad\text{iff}\quad
\text{for all } \omega \text{ with } (\nu,\omega)\in\llbracket\alpha\rrbracket,
\ \omega \models \phi,
\]
and
\[
\nu \models \langle\alpha\rangle\phi
\quad\text{iff}\quad
\text{there exists } \omega \text{ with } (\nu,\omega)\in\llbracket\alpha\rrbracket
\text{ such that } \omega \models \phi.
\]

Thus, \([\alpha]\phi\) expresses a safety-style property: every execution of \(\alpha\) ends in a state satisfying \(\phi\). By contrast, \(\langle\alpha\rangle\phi\) expresses a reachability-style property: some execution of \(\alpha\) reaches a state satisfying \(\phi\). RDFdL uses these modalities to generate proof obligations for candidate transitions between RDF state regions.

\begin{table}[htb]
  \centering
  \small
  \setlength{\tabcolsep}{6pt}
  \renewcommand{\arraystretch}{0.9}
  \begin{tabularx}{\linewidth}{@{} lX @{}}
    \toprule
    \multicolumn{2}{c}{\textbf{dL Formulas ($\varphi,\psi$)}}\\
    \midrule
    Syntax    & Meaning \\
    \midrule
    $[\alpha]\varphi$            & After all runs of $\alpha$, $\varphi$ holds (safety) \\
    $\langle\alpha\rangle\varphi$& Some run of $\alpha$ reaches a state where $\varphi$ holds (liveness) \\
    $\varphi\wedge\psi$          & Conjunction (and) \\
    $\varphi\vee\psi$            & Disjunction (or) \\
    $\varphi\to\psi$             & Implication \\
    $\varphi\leftrightarrow\psi$ & Biimplication (equivalence) \\
    \midrule[\heavyrulewidth]
    \multicolumn{2}{c}{\textbf{Hybrid Programs ($\alpha,\beta$)}}\\
    \midrule
    Syntax    & Meaning \\
    \midrule
    $\alpha;\beta$        & Sequential: do $\alpha$ then $\beta$ \\
    $\alpha\cup\beta$     & Choice: execute either $\alpha$ or $\beta$ \\
    $x := t$              & Discrete assignment: set $x$ to the value of $t$ \\
    $\{x'=t,y'=s\;\&\;Q\}$& Continuous evolution: $\dot x=t,\dot y=s$ within domain $Q$ \\
    \bottomrule
  \end{tabularx}
  \caption{Syntax of dL formulas and hybrid programs}
  \label{tab:dl-combined}
\end{table}

Table \ref{tab:dl-combined} shows \textbf{parts} of the syntax of Differential Dynamic Logic (dL) formulas and hybrid programs separately. The key semantics of the dL modalities can be stated as follows \cite{Platzer18}:
the box modality $[\alpha]\varphi$ expresses that \emph{every} execution of the hybrid program \(\alpha\) terminates in a state satisfying \(\varphi\). Dually, diamond modality \(\langle\alpha\rangle\varphi\) means that there exists \emph{some} execution of \(\alpha\) which ends in a state satisfying \(\varphi\).  Here, \(\alpha\) is the hybrid program describing the system's combined discrete and continuous behavior. For example, one can specify a reachability goal as $\varphi_1 \to \langle \alpha \rangle \varphi_2$, meaning “for any initial state, if the state satisfies $\varphi_1$, there exists some execution of $\alpha$ that eventually reaches a state satisfying $\varphi_2$”.

\textbf{In our setting, the primary objective is a reachability property}: we want to verify that from the set of initial states characterized by $\varphi_1$the system, it can eventually reach a target region $\varphi_2$ via hybrid dynamics $\alpha$. In dL, a continuous evolution is written as a hybrid program of the form $\{x' = f(x) \,\&\, Q(x)\}$, where $x$ is the evolving variable such as temperature, pressure, liquid level, etc. $f(x)$ is the ODE of $x$,  $Q(x)$ is the \emph{evolution domain} constraining the states in which the ODE is allowed to evolve. More explanations are in the following sections.

However, directly verifying a formula with the diamond modality ($\langle \alpha \rangle \varphi_2$) is challenging in practice because the theorem prover KeYmaera X provides strong support for safety properties ($[\cdot]$formulas) but not for liveness properties in diamond form. To navigate this limitation, we express the reachability statement (the existence of a trajectory) in dL using a formula $[\cdot]$ that KeYmaera X can directly handle. More details are shown in section \ref{section:formal_method}.

\subsection{Hybrid Systems and Terminology}

We introduce the basic terminology for hybrid systems. A hybrid system is a system with both discrete and continuous dynamics. A \textbf{hybrid system} is a dynamical system that exhibits both continuous evolution (described by differential equations) and discrete transitions (instantaneous jumps or mode switches). In this paper, the continuous part is given by ODEs over real-valued variables, while the discrete part is given by mode changes, assignments, and guards. We use hybrid programs as the formal representation of such systems. A device mode determines which ODE system is active, and a guard determines when a discrete mode change is admissible. 

For example, in an oven model, the continuous variable \(T\) denotes temperature. The mode \(\mathsf{On}\) is associated with a heating ODE, while the mode \(\mathsf{Off}\) is associated with a cooling ODE. A mode-change guard specifies when the oven may switch between modes. RDFdL represents such modes, variables, ODEs, and guards in RDF and translates them into hybrid programs. Other key terms are introduced below:\medskip

\begin{description}
  \item[\textbf{State.}] A \emph{state} of a hybrid system represents a value or a range of the system at a given time. It includes the values of all continuous variables (e.g., temperatures, velocities) and the system's current discrete mode. In other words, a state provides a complete description of the system’s configuration at an instant. A state in RDFdL is, however, a range in the state space. We distinguish concrete dL states from RDFdL state regions. A dL state is a valuation of variables. An RDFdL state region is an RDF resource denoting a formula over such valuations. For example, the RDF resource \(ex:s_{21}\) may denote the formula \(180 \leq T \leq 200 \wedge mode=\mathsf{On}\). Thus, RDFdL state regions are symbolic abstractions of sets of dL states, not individual dL states.
  
  \item[\textbf{Mode.}] A \emph{mode} (also called a discrete mode or control location) is a discrete condition or operational regime in which the system can operate. The mode determines which continuous dynamics (ODEs) govern the system. For example, an oven might have two modes: \textsf{On} (heating) and \textsf{Off} (cooling). Each mode is associated with different behavior (e.g., heating causes the temperature to rise according to a certain ODE, while in the Off mode, the temperature might cool down).
  
  \item[\textbf{Continuous Transition.}] A continuous transition refers to the system’s evolution over a period of time according to an ODE. In a given mode, the continuous state variables change continuously (e.g., temperature rises over time). This continuous change is modeled by differential equations $\dot{x} = f(x)$, and is often constrained by a \emph{Domain Constraint}. A continuous transition represents a smooth trajectory of states over time.
  
  \item[\textbf{Discrete Transition.}] A discrete transition is an instantaneous change that updates the system state, typically switching modes or changing certain state variables abruptly. Discrete transitions can be thought of as “jumps” or events, such as a controller turning the oven from Off to On or a sudden reset of a variable. In a hybrid program, discrete transitions are represented by statements that cause immediate changes, such as assignments (e.g., $x := \theta$)), and tests or conditionals that cause immediate choices.

  \item[\textbf{ODE.}] An \emph{ODE} (ordinary differential equation) describes the continuous dynamics of the system by specifying the rate of change of continuous variables. For example, $\dot{x} = f(x)$ is an ODE that describes how $x$ changes over time as a function of its current state. In hybrid systems, each mode typically has an associated ODE (or system of ODEs) governing the continuous evolution while the system remains in that mode. Solutions of ODEs define continuous trajectories for the state variables.
  
  \item[\textbf{Domain Constraint.}] A \emph{domain constraint} (or evolution domain) is a logical condition that restricts the region within which continuous evolution can occur. It must hold during a continuous transition. For instance, a domain constraint $Q(x)$ might require that a variable $x$ stay within a safe range while the ODE $\dot{x}=f(x)$ evolves. If the trajectory approaches a boundary that $Q$ would be violated, the continuous evolution must stop at or before that point. Domain constraints thus enforce safety conditions or physical limits.
\end{description}\medskip

\section{Motivating Example and Problem Setting}
\label{sec:motivating-example}

This section introduces the problem setting through a small oven example. The example is intentionally simple. Its role is not to demonstrate an industrial-scale controller but to make clear what information is available before RDFdL is applied, what kind of dynamic behavior has to be verified, and why verified dynamic behavior should become queryable as graph data.

\subsection{Running example: an oven process}

Consider an oven that is used in a production process to heat a product. The oven has two operating modes, \(\mathsf{On}\) and \(\mathsf{Off}\), and its temperature is represented by a continuous variable \(T\). When the oven is on, the temperature increases according to heating dynamics. When the oven is off, the temperature decreases according to cooling dynamics. The process is required to remain within a safe operating range; in the running example, we use \(T \leq 200\) as the safety bound.

The process also has ordinary metadata. The oven is a device in a production setting. It may be controlled by a button, belong to a process unit, be associated with a product batch, and be maintained by a responsible technician. This information is naturally represented as RDF graph data. Such metadata is not part of the continuous dynamics, but it is essential for answering operational questions about the process, for example, which technician is responsible for the device involved in a verified mode switch.

The dynamic behavior is represented using symbolic \emph{state regions}. A state region is not a single concrete temperature value. It describes a set of possible physical states by combining constraints over continuous variables with discrete mode information. In the oven example, we use four regions:

\[
\begin{array}{lll}
s_{11}: & T < 180, & mode=\mathsf{Off}, \\
s_{12}: & T < 180, & mode=\mathsf{On}, \\
s_{21}: & 180 \leq T \leq 200, & mode=\mathsf{On}, \\
s_{22}: & 180 \leq T \leq 200, & mode=\mathsf{Off}.
\end{array}
\]

These regions form a finite abstraction of the continuous temperature space. The boundary at \(T=180\) separates the low-temperature and high-temperature regions used by the controller, while \(T=200\) is the safety boundary. The same temperature range may occur under different modes: \(s_{11}\) and \(s_{12}\) both describe \(T < 180\), but they differ in whether the oven is off or on.

Note that the low-temperature regions are bounded strictly (\(T < 180\)), whereas the high-temperature regions include the boundary (\(180 \leq T \leq 200\)). The temperature ranges of \(s_{12}\) and \(s_{21}\) are therefore disjoint, so that, under a fixed mode, \emph{leaving} the low-temperature range and \emph{entering} the high-temperature range coincide. This disjointness is not cosmetic: the verification conditions of Section~\ref{sec:translation} rely on it, and RDFdL checks it mechanically before generating proof obligations, so that regions with overlapping boundaries are detected as modeling errors rather than silently producing unsound conclusions.

A natural candidate process path is:
\[
s_{11}
\;\xrightarrow{\mathsf{modeChange}}\;
s_{12}
\;\xrightarrow{\mathsf{next}}\;
s_{21}
\;\xrightarrow{\mathsf{modeChange}}\;
s_{22}.
\]

This path says that the oven starts in a low-temperature off region, is switched on, heats until it reaches the high-temperature region, and is then switched off. However, at this point, the path is only intended behavior. Its presence in the RDF graph does not by itself establish that the continuous transition from \(s_{12}\) to \(s_{21}\) is reachable under the heating ODE, nor that the mode changes are admissible under the modeled guards.

\paragraph*{Properties to be verified.}
For this candidate path, RDFdL has to establish three kinds of properties. First, for each mode-change edge, \((s_{11},s_{12})\) and \((s_{21},s_{22})\), the guards modeled for the oven admit the mode switch, and the target region is consistent with the new mode. Second, for the continuous edge \((s_{12},s_{21})\), that the high-temperature region \(s_{21}\) is reachable from \(s_{12}\) under the heating ODE while the evolution respects the domain constraint \(T \leq 200\). Third, globally, that every transition materialized in the verified graph preserves the safety invariant \(T \leq 200\). Section~\ref{sec:translation} turns each of these informal requirements into a precise dL proof obligation.

This is where RDFdL is needed. RDFdL distinguishes candidate transitions from verified transitions. Candidate transitions are represented as graph data, but they are not treated as verified dynamic facts. RDFdL checks them by generating dL proof obligations.

If the corresponding obligation is discharged with the required verdict (made precise in Section~\ref{sec:translation}), the transition becomes part of the verified transition graph as an entailed triple and can be queried together with the original RDF metadata.

SHACL has a separate role in this example. It is used to validate runtime observations against the shapes associated with state regions. For instance, an observation saying that the oven is currently off and has temperature \(T=185\) can be validated against the shape associated with \(s_{22}\). This validation identifies the current symbolic region. It does not prove that \(s_{22}\) is reachable; reachability and safety are handled by dL proof obligations.

The complete RDF model, SHACL validation shapes, generated dL proof obligations, and verification results for the oven are reported later as part of the evaluation. In this section, the oven only serves to introduce the problem.

\subsection{Problem setting}

The running example can be generalized as follows. RDFdL assumes an RDF description of a cyber-physical process. The graph contains ordinary metadata, such as devices, sensors, operators, products, and process units. It also contains a finite abstraction of the dynamic behavior, including state regions, modes, variables, ODE records, guards, and candidate transitions.

A state region is represented as an RDF resource, but it denotes a mathematical condition over continuous variables and discrete modes. For example, the region \(s_{12}\) denotes the condition

\[
T < 180 \wedge mode=\mathsf{On}.
\]
RDFdL therefore distinguishes the graph node that names a region from the formula denoted by that region.

A candidate transition is an intended transition between two state regions. It may come from a state diagram, a simulation model, a simulation trace, an engineering design, or a manually specified process plan. RDFdL does not synthesize such candidate transitions from scratch. Its purpose is to verify whether the proposed transitions are justified by the hybrid-system model and then expose the verified results as graph data.

For a continuous candidate transition, RDFdL checks whether the target region is reachable from the source region under the ODE associated with the current mode and within the specified evolution domain. For a mode-change candidate transition, RDFdL checks whether the relevant guard conditions allow the change of mode and whether the target region is consistent with the new mode. These checks are expressed as dL proof obligations and delegated to a dL theorem prover.

The problem addressed in this paper is therefore:

\begin{quote}
Given an RDF description of a cyber-physical process with symbolic state regions, ODE records, guards, metadata, and candidate transitions, construct a verified RDF transition graph in which dynamic transitions are materialized only when justified by dL proof obligations, and make these verified transitions queryable together with the original graph metadata.
\end{quote}

This problem setting explains why neither RDF-only nor dL-only reasoning is sufficient. RDF-only reasoning can follow metadata links and candidate paths, but it cannot prove ODE reachability or safety. dL-only verification can prove hybrid-system properties, but it does not provide graph-level integration with metadata, provenance, and SPARQL queries. RDFdL connects the two by turning verified hybrid behavior into graph data.

\section{Approach}
\label{sec:approach}

In this section, we give an overview of the roles involved and the core
components of our approach, along with their interaction when applied; see
Figure~\ref{fig:Flowchart_rdf}. We then describe our data modeling as the architectural core of RDFdL and explain how dL verification is embedded in RDF inference so that verified transitions become \emph{entailed} triples.

\subsection{Overview}

The raw inputs to our approach are (i)~a simulation model of the devices,
exported as a Functional Mock-up Unit (FMU) together with simulation traces, and (ii)~metadata authored by engineers. The state diagram over symbolic state regions is either supplied directly by the engineer or constructed by our tooling from the FMU variable metadata and trace-derived operating ranges; in both cases, it enters RDFdL as ordinary graph data\footnote{Tools such as OpenModelica can export explicit \cite{openmodelica_solve} or flattened \cite{modelica_dae_export} models. We provide code to extract RDF describing the ODEs from such exports.}. In our implementation, the pipeline reads the Functional Mock-up Interface (FMI) \texttt{modelDescription.xml} together with simulation traces (CSV) to obtain variable metadata and typical operating ranges\footnote{The FMI standard is used to export simulation models as Functional Mock-up Units (FMUs); standard FMI tooling produces \texttt{modelDescription.xml} and runs FMU simulations; see the FMI~2.0 specification and reference implementations~\cite{fmi2spec,referencefmus,fmpy}.}. In a state diagram, each state region consists of a numerical range for each physical variable (e.g., temperature $T$, pressure $P$, tank level $h_1$) and a configuration of the involved devices' modes, for example, whether they are
on or off. In our oven example, a state region might be "$T<180,\ \text{oven}=\textsf{On}$" to model the oven heating but remaining below the temperature for the Maillard reaction, which is important when making cake. State regions and their value ranges are represented as SHACL constraints. Next to the simulation models and the state diagram, metadata and master data about the involved devices, their parameters, and the setup are recorded in an RDF graph, e.\,g.\ by the engineers who designed the device, the manufacturer, or the engineers who set up the device on a shop floor. In our example, this may include the device's make and model, how the buttons are mapped to mode changes, the physical layout of the shop floor, and the go-to technician for technical problems.

A distinguishing design decision of RDFdL is \emph{where} verification
happens. A loosely coupled design would verify offline and import the results
as ordinary triples; the graph would then merely \emph{store} verification
output, and nothing would tie the stored triples to the proofs that justify
them. RDFdL instead registers the dL prover inside the RDF inference
machinery: the transition predicates: \texttt{:next} and \texttt{:modeChange}
are \emph{defined} by inference rules whose evaluation triggers dL proofs
(Section~\ref{subsec:entailment}). As a consequence, a verified transition
holds in the inference model if and only if the corresponding proof obligation is discharged with the required verdict by construction, not by convention. To SPARQL, entailed transitions are indistinguishable from asserted triples; in particular, property-path queries range over verified behavior (cf.\ the discussion of entailment and property paths in
Section~\ref{sec:preliminaries}). A SPARQL engine can thus retrieve static
information about devices, configurations, products, and fixed processes and, in the same query, dynamic information about paths over the state diagram, e.g., whether there exists a verified path between a certain state region and the successful end of the process or which state regions lie along a "happy path." All this is shown in Figure~\ref{fig:Flowchart_rdf}.

\begin{figure}[t]
  \centering
\includegraphics[width=1\linewidth]{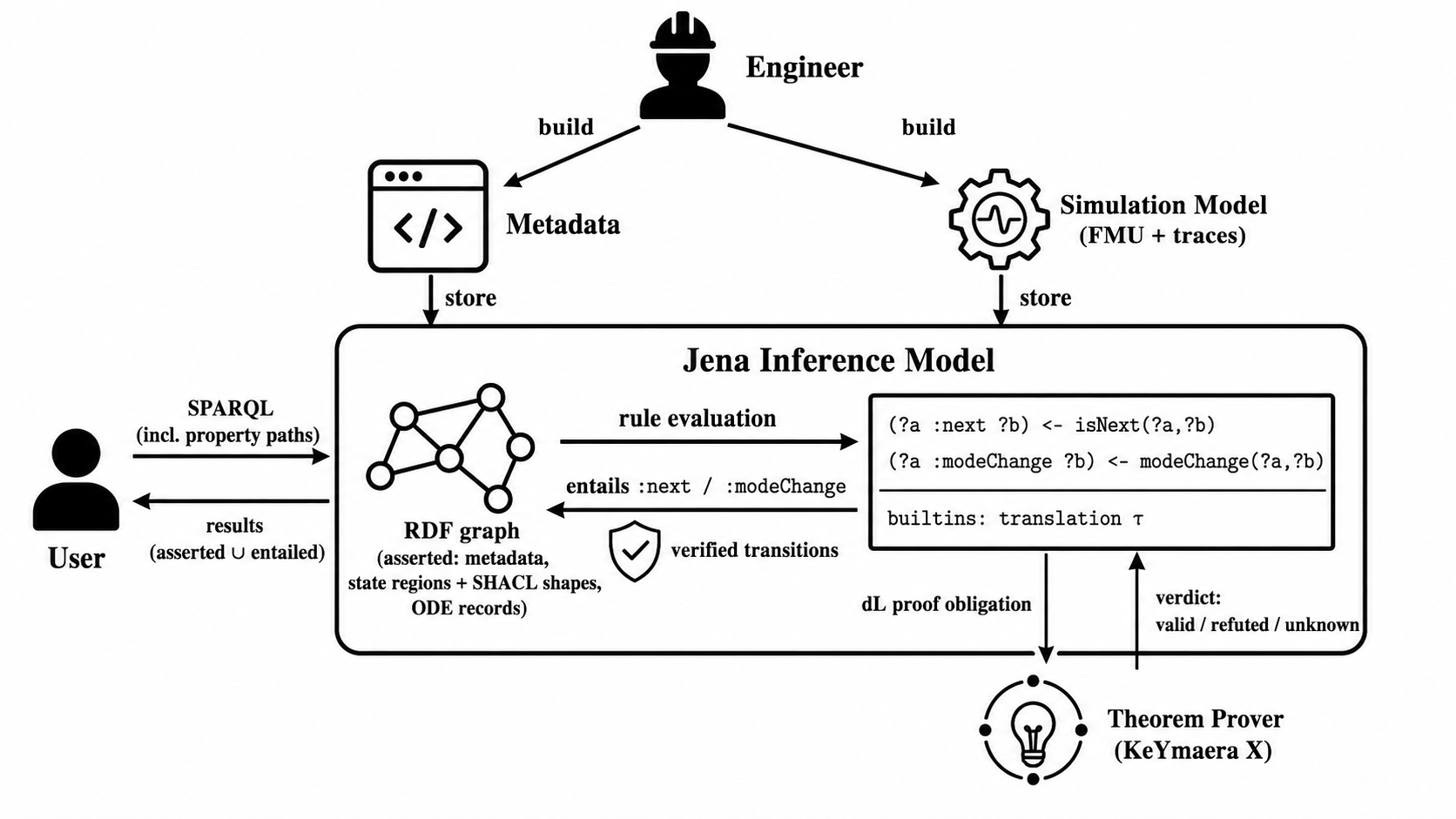}
  \caption{RDFdL architecture. The transition predicates \texttt{:next} and
\texttt{:modeChange} are defined by Jena inference rules whose builtins
invoke the KeYmaera~X prover; verified transitions are therefore
\emph{entailed} triples, visible to SPARQL queries, including property
paths alongside asserted data.}
  \label{fig:Flowchart_rdf} 
\end{figure}

\subsection{Data Modeling Ontology Design: Devices, Modes, and ODEs}
\label{subsec:ontology}

We develop an RDF vocabulary to formally describe the dynamic part of the system; its main classes and properties are shown in Figure~\ref{fig:Ontology}. The vocabulary is published at \url{https://purl.archive.org/rdfdl/vocab}\footnote{Resolvable as Turtle.}, and all listings in this paper use it. Following linked-data practice, the vocabulary imports and reuses existing terms wherever possible, and for the dL-specific part: Devices are modeled as \texttt{ssn:System} (SSN), state regions, and workflow structure reuse are \texttt{wild:State} and \texttt{wild:WorkflowModel} from the WiLD workflow vocabulary~\cite{DBLP:conf/semweb/KaferH18}, actuators come from SOSA, agents and provenance from PROV, and ordinary master data organizations, processes, products, and operators reuse \texttt{schema.org} terms.

The dL-specific contains four classes and a dozen properties. A \texttt{rdfdl:ModeRecord} associates a device (\texttt{rdfdl:hasDevice}) and one of its modes (\texttt{rdfdl:hasMode}, range \texttt{rdfdl:Mode}) with the ODE governing the continuous evolution in that mode (\texttt{rdfdl:hasODE}, range \texttt{rdfdl:ODE}; Section~\ref{ssec:ODE}). Each state region links to the SHACL node shape carrying its numeric bounds via \texttt{rdfdl:hasShape} (Section~\ref{subsec:shacl-shapes}), and a \texttt{wild:WorkflowModel} designates its \texttt{rdfdl:initialState}. Continuous variables are first-class resources of type \texttt{rdfdl:Variable}; each carries a \texttt{rdfdl:symbol}, the unique identifier under which it occurs in the
mathematical expression strings of any ODE that declares it via \texttt{rdfdl:evolvingVariable}. This property is the formal link between variables as RDF resources (e.g., \texttt{ex:T}) and their occurrences inside expressions (e.g., \texttt{"x"}); see
Section~\ref{section:formal_method}.

Verified transitions are represented by \texttt{rdfdl:next} and
\texttt{rdfdl:modeChange}, both declared as subproperties of
\texttt{rdfdl:transition}; under RDFS entailment, a single property path
over \texttt{rdfdl:transition}, therefore, traverses both kinds of verified
edges. Finally, the vocabulary is self-describing: it ships SHACL shapes for its own terms, e.g., \texttt{rdfdl:ODEShape} requires every ODE record to carry exactly one derivative, one starting and one ending condition, one
evolution-domain constraint, and one evolving variable so that RDFdL
Inputs are themselves machine-verifiable before any verification is
attempted.

\begin{figure}[t]
    \centering
    \includegraphics[width=0.8\textwidth]{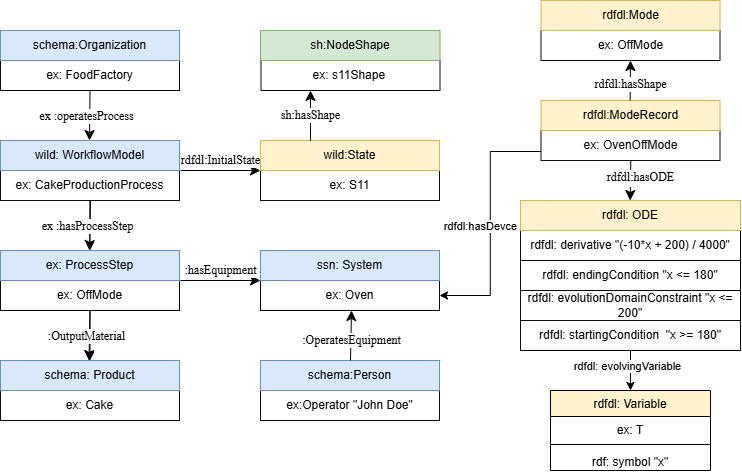}
    \caption{Main Classes and Properties of our Ontology.}
    \label{fig:Ontology} 
\end{figure}

\subsection{Integrating ODE Models}

Cyber‑physical behavior is represented in RDF by attaching \emph{ordinary differential equations (ODEs)} to the device modes. Each \texttt{:ModeRecord} links the discrete controller state to a resource of type \texttt{:ODE}, characterized using the following:
\begin{itemize}
  \item \texttt{rdfdl:derivative} \(\dot{x}=f(x)\),
  \item \texttt{rdfdl:starting/endingCondition}: guards that decide when the mode may start or must stop;
  \item \texttt{rdfdl:evolutionDomainConstraint}: an invariant $Q$ that must hold during the continuous evolution;
  \item \texttt{rdfdl:evolvingVariable}: the state variable governed by the ODE.
\end{itemize}
This turns the knowledge graph into a \emph{hybrid system model}: triples select a discrete mode, while the attached ODE governs the continuous evolution between two state transitions. The RDF for our example is listed in the Appendix \ref{ssec:ODE}.

\paragraph*{Expression strings and variable alignment.}
The expression strings used in \texttt{rdfdl:derivative}, the guard conditions, and the evolution domain constraint follow a small arithmetic grammar given in Appendix~\ref{ssec:ODE}. The identifiers occurring in these strings are resolved against the \texttt{ssn:symbol} values of the variables that the ODE declares via \texttt{rdfdl:evolvingVariable}: Every identifier in an expression must equal the symbol of exactly one declared variable, and symbols must be unique per ODE. The RDFdL pipeline checks this contract mechanically before translation and rejects, with an error report, any ODE record whose expressions mention undeclared identifiers. Listing~\ref{lst:ode-rdf} in Appendix~\ref{ssec:ODE} shows the oven's \texttt{rdfdl:ModeRecord}: the variable\texttt{ex:T} carries \texttt{ssn:symbol "x"}, which is why the derivative string may legally mention \texttt{x}.

This turns the knowledge graph into a \emph{hybrid system model}: triples select a discrete mode, while the attached ODE governs the continuous evolution between two state regions. The RDF for our example is listed in Appendix~\ref{ssec:ODE}.

\subsection{Verification as Entailment}
\label{subsec:entailment}
 
RDFdL implements the transition predicates as inference rules of the Jena rule engine, which applies the dL prover for reasoning:
 
\begin{listing}[h]
  \scriptsize
  \captionsetup{skip=2pt}
  \caption{Inference rules defining the verified transition predicates.}
  \label{lst:rules}
  \begin{lstlisting}[frame=single, basicstyle=\ttfamily\scriptsize,
                     numbers=none, breaklines=true]
[rule1: (?a :next ?b)       <- (?a rdf:type :State),
                               (?b rdf:type :State),
                               isNext(?a, ?b) ]
[rule2: (?a :modeChange ?b) <- (?a rdf:type :State),
                               (?b rdf:type :State),
                               modeChange(?a, ?b) ]
  \end{lstlisting}
\end{listing}
 
When the rule engine evaluates the rule \isNext{} or \modeChange{} on a pair of state-region resources, it extracts the state predicates and dynamics of both regions from the graph (Section~\ref{subsec:translation}), and the corresponding dL proof obligation; then it submits to KeYmaera~X over its REST interface and succeeds if and only if the obligation is discharged with the verdict required by the decision rules of Section~\ref{subsec:verification}. Verdicts are memoized per state pair, so each obligation is proved at most once per session.
 
Because entailed and asserted triples are treated uniformly during query evaluation in Jena (Section~\ref{sec:preliminaries}), SPARQL property-path queries such as \texttt{ex:s0 (:next|:modeChange)+ ex:s4} range over \emph{verified} behavior. This answers the question of whether theorem-prover calls can be included in SPARQL query evaluation: in RDFdL, they are part of the entailment regime under which queries are answered.

\section{Extracting and Proving State Transitions}
\label{section:formal_method}
 
The verification pipeline has three automatic phases: \emph{(i)} \textbf{extraction} of dL specifications from RDF, \emph{(ii)} mechanical \textbf{proof} of those specifications with KeYmaera~X, and \emph{(iii)} \textbf{interpretation} of the prover's verdict by the decision rules that define the transition predicates.

\subsection{SHACL Shapes for State Regions and Runtime Validity}
\label{subsec:shacl-shapes}

SHACL shapes play two roles in RDFdL. First, they are the syntax in which the numeric bounds of state regions are recorded: the translation of Section~\ref{subsec:translation} reads the constraint components of a region's shape to construct its state predicate. Second, they are used, under standard W3C validation semantics, to classify runtime observations: an incoming observation is validated against the shapes of all state regions to identify the symbolic region it belongs to. The semantics of dynamics are given by the dL proof obligations.
 
We define a SHACL NodeShape with the expected sensor readings and actuator modes for every predefined state region in the model. Each NodeShape serves as a semantic contract, specifying precisely which conditions must hold in the RDF data for the system to be in that region. These shapes serve as a type-checking layer for dynamic data: any incoming real-time state must conform to one of these shapes to be recognized as a valid system state.
 
Each NodeShape uses SHACL property constraints, modeled as blank node entries under \texttt{sh:property}, to capture the requirements for each system parameter. In the oven model, the region \texttt{s11} is characterized by two conditions: the oven is in \texttt{ex:OffMode}, and the temperature is strictly below the threshold (\texttt{x < 180}). We implement these constraints in a shape called \texttt{s11Shape}.
 
In our case study, when a real-time state with \texttt{mode=ex:OffMode} and measured temperature \texttt{T=150} is observed, the SHACL validator accepts it as conforming to \texttt{s11Shape}. The listing below shows the generated real-time RDF instance used for Scenario~2.
 
\begin{listing}[h]
  \scriptsize
  \captionsetup{skip=2pt}
  \caption{RDF Example for Scenario 2 Real-Time State.}
  \label{lst:rdf-example_state_s2}
  \begin{lstlisting}[frame=single, basicstyle=\ttfamily\scriptsize,
                     numbers=none, breaklines=true]
 
ex:Oven a ex:Device ;
    ex:mode ex:OffMode ;
    ex:x 150 .
 
ex:realTimeState a ex:State ;
    ex:hasOven ex:Oven .
  \end{lstlisting}
\end{listing}
 
For this scenario, the validation result is:
\begin{itemize}
\item Matched shape: \texttt{s11Shape}
\item Violated shapes: \texttt{s12Shape}, \texttt{s21Shape}, \texttt{s22Shape}
\end{itemize}
 
This indicates that the observation is a valid instance of region \texttt{s11} and conforms to the expected "off, low-temperature" operational condition.
 
\subsection{From RDF/SHACL to dL Proof Obligations}
\label{subsec:translation}
 
For every candidate \emph{state transition} \((s_{1},s_{2})\) in the
knowledge graph, we are interested in the reachability obligation,
\[
   \underbrace{\varphi_{1}}_{\text{precondition}}
   \;\longrightarrow\;
   \bigl\langle\,
        \underbrace{\alpha}_{\text{hybrid program}}
   \,\bigr\rangle\!
   \underbrace{\varphi_{2}}_{\text{postcondition}},
\]
which asks whether there exists some execution of \(\alpha\) that starts in a state satisfying \(\varphi_{1}\) and eventually reaches a state satisfying \(\varphi_{2}\). In principle, a proof of this dL formula represents reachability.
 
However, KeYmaera~X provides much stronger automation for safety-style box formulas than for diamond formulas. We therefore verify the \emph{refutability} of the following box obligation instead:
\begin{equation}
  \varphi_{1}
  \;\longrightarrow\;
  \bigl[\,\alpha \;\&\; (\varphi_{1}\vee\varphi_{2})\bigr]\,
  \varphi_{1}.
  \label{eq:isnext1}
\end{equation}

This means "all executions of \(\alpha\) that evolve only in states satisfying \(\varphi_{1}\vee\varphi_{2}\) end in a state satisfying \(\varphi_{1}\)“. The execution of \(\alpha\) is restricted to the union of the source and target regions. If Formula~\eqref{eq:isnext1} is \emph{valid}, then no execution that stays inside \(\varphi_{1}\vee\varphi_{2}\) can move from \(\varphi_{1}\) into \(\varphi_{2}\). Conversely, if~\eqref{eq:isnext1} is \emph{refutable}, there exists an execution of \(\alpha\) that starts in \(\varphi_{1}\), stays within \(\varphi_{1}\vee\varphi_{2}\), and leaves \(\varphi_{1}\); by the
well-formedness condition below, which this execution must enter \(\varphi_{2}\). Thus, refuting~\eqref{eq:isnext1} witnesses the liveness property \(\varphi_{1} \rightarrow \langle \alpha \,\&\, (\varphi_{1}\vee\varphi_{2})\rangle\, \varphi_{2}\).
 
\begin{definition}[Well-formed candidate transition]\label{def:wellformed}
A candidate continuous transition \((s_1,s_2)\) with state predicates
\(\varphi_1,\varphi_2\) is \emph{well-formed} iff
(i)~\(m(s_1)=m(s_2)\);
(ii)~\(S(s_1)\neq S(s_2)\); and
(iii)~the first-order arithmetic sentence
\[
  \forall x\;\bigl( (\varphi_1 \vee \varphi_2) \wedge \neg\varphi_1
    \;\rightarrow\; \varphi_2 \bigr)
\]
is valid, i.e., within the envelope \(\varphi_1\vee\varphi_2\), leaving
\(\varphi_1\) entails entering \(\varphi_2\). Condition~(iii) holds in
particular whenever \(\varphi_1\) and \(\varphi_2\) are disjoint and their
union has no gap, as for the oven regions of
Section~\ref{sec:motivating-example}, whose strict boundary
(\(T<180\) versus \(180\le T\le 200\)) was chosen precisely for this purpose.
RDFdL discharges condition~(iii) automatically as a quantifier-elimination
obligation before generating~\eqref{eq:isnext1}; ill-formed candidates are
rejected and reported, never materialized.
\end{definition}
 
\paragraph{Step 1: SHACL $\Rightarrow$ state predicates \(\varphi\).}
 
The mapping of SHACL constraint components to dL literals is shown in
Table~\ref{tab:factmap}; the constant \(c\) is taken directly from the
literal of the corresponding SHACL constraint. Given a state URI, the
notation used in the mapping is shown in Table~\ref{tab:notation}.
 
\begin{table}[t]
\setlength{\textfloatsep}{4pt}
  \setlength{\intextsep}{4pt}
  \centering
  \small
  \setlength{\tabcolsep}{6pt}
  \renewcommand{\arraystretch}{0.9}
  \begin{tabular}{@{} l l @{}}
    \toprule
    \textbf{SHACL constraint}   & \textbf{dL literal} \\
    \midrule
    \texttt{sh:hasValue}        & \(x = c\) \\
    \texttt{sh:minInclusive}    & \(x \ge c\) \\
    \texttt{sh:maxInclusive}    & \(x \le c\) \\
    \texttt{sh:minExclusive}    & \(x > c\) \\
    \texttt{sh:maxExclusive}    & \(x < c\) \\
    \bottomrule
  \end{tabular}
  \caption{SHACL constraints and their dL correspondents.}
  \label{tab:factmap}
\end{table}
 
\begin{table}[t]
  \centering
    \begin{tabularx}{\columnwidth}{@{}lX@{}}
        \toprule
        \textbf{Symbol}         & \textbf{Meaning} \\
        \midrule
        $p$                     & URI of a numeric evolving variable, e.g.\ \texttt{ex:T} \\
        $f, v$                  & SHACL constraint components and their literal values \\
        $\theta_f$              & Comparison operator mapped from $f$ (Table~\ref{tab:factmap}) \\
        $\mathsf{mode}_d$       & Current mode of device $d$ \\
        $M_d$                   & Finite set of modes declared for device $d$ \\
        $\mathbf{f}$            & ODE derivatives (\texttt{:derivative}) \\
        $Q(\mathbf{x})$         & Evolution-domain constraint \\
        \bottomrule
      \end{tabularx}
  \caption{Notation used in the SHACL\,$\rightarrow$\,dL mapping.}
  \label{tab:notation}
\end{table}
\vspace{-4pt}
 
\begin{enumerate}
\item Read the region's \texttt{sh:NodeShape}.
\item Convert every numeric constraint \(\langle p,f,v\rangle\) into an
 
      atomic comparison \(p\;\theta_f\;v\) using Table~\ref{tab:factmap}.
\item Convert every \texttt{sh:hasValue} on \texttt{:mode} to a literal
     
      \(\mathsf{mode}_d = m\) with \(m \in M_d\).
\end{enumerate}
Conjoin all literals; the result is the \emph{state predicate}
\(\varphi_s\). For example, for the oven region \texttt{s11},

\(\varphi_{11} = T<180 \wedge \mathsf{mode}_{\mathit{Oven}}=\mathsf{Off}\).

The SHACL source is listed in Appendix~\ref{ssec:state_in_Shacl}.
 
\paragraph{Step 2: RDF $\Rightarrow$ hybrid program \(\alpha\).}
 
Follow the \texttt{:hasMode} link of the region to its unique
\texttt{:ModeRecord}; extract
 
\smallskip
\(\displaystyle
  \alpha\;=\;
  \bigl\{
     \dot{\mathbf x}=\mathbf f(\mathbf x)
     \;\&\;Q(\mathbf x)
  \bigr\},
\)
 
\smallskip
\noindent
where \(\mathbf f\) is the vector of derivatives defined in
Section~\ref{subsec:ode-models} and \(Q\) is the evolution domain. If a mode
drives multiple variables -- as in the drum-boiler case, where the flowing
mode governs both steam temperature \(T_S\) and pressure \(P_S\) -- all ODEs
are listed in the same evolution block, i.e., a single
\(\{\dots \& \dots\}\) statement, and their domain constraints are conjoined;
this respects the hybrid program syntax of
Section~\ref{sec:preliminaries}.
 
\paragraph{Step 3: Assemble the obligation.}
Set \(\varphi_{1}=\varphi_{s_{1}}\), \(\varphi_{2}=\varphi_{s_{2}}\), check
well-formedness (Definition~\ref{def:wellformed}), and build
Formula~\eqref{eq:isnext1}. Each parameter of the dL obligation can be traced
back to the input RDF/SHACL graph, allowing \textbf{explainable proofs}.
 
\smallskip
\emph{Mapping Example.}\;

For the oven, \(s_{11}\) is the region \(T<180\) with the oven off;
\(s_{12}\) is \(T<180\) with the oven on;
\(s_{21}\) is \(180\le T\le 200\) with the oven on; and
\(s_{22}\) is \(180\le T\le 200\) with the oven off.
 
Check \texttt{s12\,$\rightarrow$\,s21}:
\(\varphi_{12}: T<180\),\;
\(\alpha_{12}:\{\dot T = 0.605 - 0.0025\,(T-20)\;\&\;T\le200\}\),\;
\(\varphi_{21}: 180\le T\le200\), yielding the obligation
$\varphi_{12}\;\longrightarrow\;
 [\alpha_{12}\;\&\;(\varphi_{12}\lor\varphi_{21})]\,\varphi_{12}$.
 
\subsection{Verification in dL}
\label{subsec:verification}
 
We formalize transition rules that classify how the hybrid system performs
state transitions. Each state region is represented as an RDF node, encoding
the mode configuration of each device (e.g., heater, tank, etc.) along with
continuous numeric bounds (e.g., temperature ranges) for the physical
variables in that region. In this section, we present the \isNext{} and
\modeChange{} rules; these are exactly the builtins evaluated by the
inference rules of Section~\ref{subsec:entailment}.
 
\subsubsection{Rule: \isNext}
\label{subsub:isnext}
 
\begin{definition}[RDF Extraction]\label{def:state}
Let $D$ be the finite set of device resources and, for each $d\in D$, let
$M_d$ be the finite set of mode resources declared for $d$. For each state
region resource $s$, extract:
\[
\begin{aligned}
  &\textbf{Mode }m(s) = \bigl(m_{d}(s)\bigr)_{d\in D},
    \quad m_d(s)\in M_d,\\
  &\textbf{Region }S(s) =
    \{\,x\in\R^n \mid \underline x_i(s)\mathrel{\lhd_i} x_i
      \mathrel{\rhd_i}\overline x_i(s),\ 1\le i\le n\},
\end{aligned}
\]
where the bounds and the strictness of the comparisons
$\lhd_i,\rhd_i\in\{<,\le\}$ are read off the SHACL constraint components of
the shape associated with $s$ via Table~\ref{tab:factmap}. Then define the
predicate of $s$:
\[
  \varphi_s(x) \;=\; (x\in S(s)) \;\wedge\;
  \bigwedge_{d\in D}\bigl(\mode_d=m_d(s)\bigr).
\]
\end{definition}
 
\begin{definition}[\isNext]\label{def:isnext}
Given two state regions $s_{1},s_{2}$, let
\(\varphi_{1}=\varphi_{s_{1}}\) and \(\varphi_{2}=\varphi_{s_{2}}\).
Then \(\isNext(s_{1},s_{2})\) holds iff:
\begin{enumerate}
  \item $m(s_{1})=m(s_{2})$ \quad (modes of all devices are the same);
  \item $S(s_{1})\neq S(s_{2})$ \quad (the ranges differ); and
  \item \emph{Reachability within the envelope:} there exist a state
        \(\nu\models\varphi_1\) and an execution of the shared-mode dynamics
        \(\alpha\) that starts in \(\nu\), remains within
        \(S(s_1)\cup S(s_2)\) throughout, and ends in a state satisfying
        \(\varphi_2\).
\end{enumerate}
\end{definition}

\begin{remark}
Restricting condition~3 to the envelope \(S(s_1)\cup S(s_2)\) is deliberate:
\isNext{} models a \emph{direct} transition between adjacent regions.
Behavior that must pass through other regions is not a single edge; it is a
path, and paths are exactly what SPARQL property-path queries over the
verified transition graph recover.
\end{remark}
 
\paragraph{dL Verification Condition}
Let $m=m(s_{1})=m(s_{2})$ and extract from the RDF ODE records the hybrid
program
\[
  \alpha \;=\;\{\,\dot x = f_{m}(x)\;\&\;Q_{m}(x)\,\},
\]
where $Q_{m}(x)$ is the conjunction of all domain constraints for the mode

$m$ (for the oven, \(T\le200\)). The proof obligation is
Formula~\eqref{eq:isnext1}, instantiated with \(\varphi_1,\varphi_2\), and
$\alpha$.

\paragraph{dL Decision Rule}

For each obligation, the prover run yields one of three verdicts:
\begin{itemize}
  \item If \eqref{eq:isnext1} is \textbf{valid}, no trajectory within the
        envelope leads from $\varphi_{1}$ to $\varphi_{2}$, so
        \(\isNext(s_{1},s_{2})=\mathsf{false}\).
  \item If it is \textbf{refutable}, there is a counterexample trajectory
        that exits $\varphi_{1}$ while staying within
        $\varphi_{1}\vee\varphi_{2}$; by
        Definition~\ref{def:wellformed}(iii) it must enter $\varphi_{2}$, so
        \(\isNext(s_{1},s_{2})=\mathsf{true}\) and the triple is entailed.
  \item If the proof attempt is \textbf{inconclusive} (the automation
        neither closes the proof nor produces a refutation within its
        resource bounds), RDFdL conservatively does \emph{not} materialize
        the transition and records the obligation as an open proof artifact.
\end{itemize}
The refutation in the second case is genuine semantic refutation, not
negation-as-failure: KeYmaera~X closes the attempt with a concrete
counterexample trajectory, which is exactly the existential witness required
by Definition~\ref{def:isnext}.

\subsubsection{Example 1 (Oven: $s_{12}\!\to\!s_{21}$)}

\[
\begin{aligned}
  \varphi_{12} &:\; x<180,\\
  \alpha_{12} &:\;
      \{\,x' = 0.605-0.0025(x-20)\;\&\;x\le200\,\},\\
  \varphi_{21} &:\; 180\le x\le200,
\end{aligned}
\]
where $x$ is the oven temperature and the mode is
$\mode_{\textit{Oven}}=\mathsf{On}$ (heater on). The dynamics admit the
closed-form solutions \(x(t) = 262 - (262-x_0)\,e^{-0.0025\,t}\) for initial
value \(x_0\), which increases strictly monotonically towards the equilibrium
\(262\). Hence, from any \(x_0<180\), the trajectory crosses \(180\) in
finite time while still satisfying the domain constraint \(x\le200\); for
instance, from \(x_0=179\) the solution
\(x(t)=262-83\,e^{-0.0025\,t}\) reaches \(x=180\) at
\(t=\ln(83/82)/0.0025\approx 4.85\). When the obligation is submitted to the
proof engine, the prover accordingly returns such a counterexample
trajectory: the obligation is \emph{refutable},
\(\isNext(s_{12},s_{21})=\mathsf{true}\), and the following triple is
entailed:
 
\smallskip
\quad\quad\quad\texttt{ex:s12 :next ex:s21}.
 
\subsubsection{Rule: \modeChange}
\label{subsub:modechange}

\begin{definition}[\modeChange]\label{def:modechange}
For two state regions $s_{1},s_{2}$ with state predicates
$\varphi_{1},\varphi_{2}$, \(\modeChange(s_{1},s_{2})\) holds iff:
\begin{enumerate}
  \item $S(s_{1}) = S(s_{2})$ \quad (the ranges are the same);
  \item $m(s_{1}) \neq m(s_{2})$ \quad (at least one device mode flips; each
        such device is a flipped device $d^\star$);
  \item every flipped device $d^\star$ satisfies
        Formula~\eqref{eq:modechange-start} and
        Formula~\eqref{eq:modechange-link} below.
\end{enumerate}
\end{definition}
 
\paragraph{Per-device obligations}\label{def:mc-oblig}
Let $d^\star$ be a flipped device with old mode
$m_{\mathrm{old}}=m_{d^\star}(s_{1})$ and new mode
$m_{\mathrm{new}}=m_{d^\star}(s_{2})$. From the corresponding
\texttt{:ModeRecord}s we extract the starting condition
$SC_{m_{\mathrm{new}}}(x)$ of the new mode and the ending condition
$EC_{m_{\mathrm{old}}}(x)$ of the old mode. KeYmaera~X must show both
obligations valid:
\begin{align}
  \varphi_{1} &\;\longrightarrow\; SC_{m_{\mathrm{new}}}(x),
  \label{eq:modechange-start} \\[4pt]
  EC_{m_{\mathrm{old}}}(x) &\;\longleftrightarrow\; SC_{m_{\mathrm{new}}}(x).
  \label{eq:modechange-link}
\end{align}

Obligation~\eqref{eq:modechange-start} ensures the new mode may start from
anywhere in the source region; obligation~\eqref{eq:modechange-link} ensures
the old mode ends exactly where the new mode may begin, so the switch occurs
on a well-defined guard surface. If any flipped device fails either obligation, then
\(\modeChange(s_{1},s_{2})=\mathsf{false}\). As for \isNext{}, an
inconclusive prover run is treated conservatively: the transition is not
materialized.
 
\paragraph{Example $\,s_{11}\!\to\!s_{12}$}

Region $S(s_{11})\colon T<180$; flipped device: Oven
$(\mathsf{Off}\to\mathsf{On})$. From RDF we obtain
$SC_{\text{On}}\colon T\le180$ and $EC_{\text{Off}}\colon T\le180$.
\[
\begin{aligned}
\text{F-1}&:\ T<180\to T\le180\;(\text{valid})\\
\text{F-2}&:\ T\le180\leftrightarrow T\le180\;(\text{valid})
\end{aligned}
\]
Both obligations hold; therefore, the following triple is entailed:
\[
  \texttt{ex:s11}\;\texttt{ex:Oven\_OffMode\_to\_OnMode}\;\texttt{ex:s12}.
\]

\subsection{Guarantees}
\label{subsec:guarantees}
 
We now establish that the decision rules above are sound with respect to the
hybrid-system semantics, and we characterize the sense in which they are
incomplete.
 
\begin{lemma}[Continuous-edge soundness]\label{lem:isnext-sound}
Let \((s_1,s_2)\) be a well-formed candidate transition (Definition~\ref{def:wellformed}) with shared mode configuration $m$ and
\(\alpha=\{\dot x=f_m(x)\,\&\,Q_m(x)\}\). If Formula~\eqref{eq:isnext1} is refutable, then \(\isNext(s_1,s_2)\) holds in
the sense of Definition~\ref{def:isnext}.
\end{lemma}
\begin{proof}
Refutability of~\eqref{eq:isnext1} yields a state \(\nu\models\varphi_1\) and an execution \((\nu,\omega)\in\llbracket\alpha\,\&\,(\varphi_1\vee\varphi_2)\rrbracket\) with \(\omega\not\models\varphi_1\). By the semantics of the evolution domain constraints (Section~\ref{sec:preliminaries}), every state along this execution, in particular \(\omega\) satisfies \(\varphi_1\vee\varphi_2\); hence, the execution remains within \(S(s_1)\cup S(s_2)\) throughout, as required by Definition~\ref{def:isnext}(3). Moreover, \(\omega\models(\varphi_1\vee\varphi_2)\wedge\neg\varphi_1\), so by well-formedness condition~(iii), \(\omega\models\varphi_2\). Thus \(\nu\) and the execution witness Definition~\ref{def:isnext}(3).
\end{proof}
 
\begin{lemma}[Discrete-edge soundness]\label{lem:modechange-sound}
Let \((s_1,s_2)\) satisfy conditions~1 and~2 of
Definition~\ref{def:modechange}, and let every flipped device satisfy
Formulas~\eqref{eq:modechange-start} and~\eqref{eq:modechange-link}. Then the
instantaneous mode switch maps every state satisfying \(\varphi_1\) to a
state satisfying \(\varphi_2\).
\end{lemma}
\begin{proof}
Let \(\nu\models\varphi_1\). The switch update only the discrete mode
variables of the flipped devices and leave the continuous valuation
unchanged; call the resulting state \(\nu'\). Since
\(S(s_1)=S(s_2)\) and \(\nu\in S(s_1)\), we have \(\nu'\in S(s_2)\). For every
device $d$ that is not flipped, \(m_d(s_2)=m_d(s_1)\) and the mode is
unchanged; for every flipped device, the switch sets its mode to
\(m_{d^\star}(s_2)\) by construction. Hence
\(\nu'\models\varphi_2\). Formula~\eqref{eq:modechange-start} guarantees that
the switch is admissible from \(\nu\), because
\(\nu\models\varphi_1\rightarrow SC_{m_{\mathrm{new}}}\); and
Formula~\eqref{eq:modechange-link} guarantees that the ending condition of
the old mode holds exactly then, so the old mode may terminate at the switch
point.
\end{proof}
 
\begin{corollary}[Global path safety]\label{cor:path-safety}
Let \(G=(S,E)\) be the verified transition graph, where
\((s_i,s_j)\in E\) iff \(\isNext(s_i,s_j)\) or \(\modeChange(s_i,s_j)\) was
established by the decision rules. If every state region \(s\in S\) entails a
global invariant \(\mathrm{Inv}\), i.e.\ \(\varphi_s\models\mathrm{Inv}\) for
all \(s\), then every finite path in \(G\) preserves \(\mathrm{Inv}\).
\end{corollary}
\begin{proof}
By induction on the path length. The base case is immediate from
\(\varphi_{s_0}\models\mathrm{Inv}\). For the step, consider an edge
\((s_i,s_{i+1})\in E\). If it is a continuous edge, then by
Lemma~\ref{lem:isnext-sound} the witnessing execution remains within
\(S(s_i)\cup S(s_{i+1})\), and both \(\varphi_{s_i}\) and
\(\varphi_{s_{i+1}}\) entail \(\mathrm{Inv}\); hence, every state along the
execution satisfies \(\mathrm{Inv}\). If it is a discrete edge, then by
Lemma~\ref{lem:modechange-sound} the switch maps \(\varphi_{s_i}\)-states to
\(\varphi_{s_{i+1}}\)-states instantaneously, and both entail
\(\mathrm{Inv}\).
\end{proof}
 
For the oven, every region entails \(T\le200\); hence, by
Corollary~\ref{cor:path-safety}, every verified path -- in particular, the
happy path of Section~\ref{sec:motivating-example} -- preserves the safety
bound. This discharges the third verification goal stated there.

\paragraph*{Incompleteness.}

RDFdL is \emph{sound but incomplete}: every entailed \texttt{:next} or
\texttt{:modeChange} triple corresponds to verified hybrid behavior
(Lemmas~\ref{lem:isnext-sound} and~\ref{lem:modechange-sound}), but the
absence of a triple does not prove the absence of behavior. A transition may
fail to be materialized because the obligation is genuinely valid (the
behavior is impossible), because the prover's automation returned an
inconclusive verdict on non-polynomial arithmetic, or because a resource
bound was exceeded. All queries over the verified transition graph are
therefore to be read under this sound-approximation semantics; the rate of
inconclusive verdicts is reported as part of the evaluation
(Section~\ref{sec:evaluation}).

\section{Evaluation}
\label{sec:Evaluation}

We evaluate our approach from a logical perspective and its scalability and applicability.

\subsection{Evaluation setup}

\subsubsection{Case Study}
\label{subsec:case_study}

To illustrate the generality of our RDFdL workflow, we modeled four representative cyber–physical processes. A single-device thermostat controller (oven), a two-tank flow system with a two-dimensional non-linear ODE, a multi-device yogurt processing example with more sequential parallel modes, and a Drumboiler steam generation system with a two-dimensional linearized ODE. Figure \ref{fig:case_study1} shows the simulation model, state diagrams, and ODE of the oven and empty tank example. The oven case is the running example introduced in Section~\ref{sec:motivating-example}.\textbf{ In the evaluation, we provide the complete RDF graph, SHACL validation shapes, ODE records for the \(\mathsf{On}\) and \(\mathsf{Off}\) modes, generated dL proof obligations, and verified transition graph.} The oven is a single-device thermostat-style controller with one continuous variable, temperature \(T\), and two operating modes. The intended behavior is to heat from a low-temperature region to a target region while respecting the safety bound \(T \leq 200\). and the \emph{Empty‑Tank} models interconnected storage tanks taken from the standard \textsc{OpenModelica} “\texttt{TankSystem}”. Each tank is equipped with an ultrasonic level sensor and a valve that allows water to flow from the upper vessel (Tank 1) to the lower vessel (Tank 2). Figure \ref{fig:case_study2} shows yogurt and Drumboiler examples. The yogurt example shows a simplified yogurt production pipeline \cite{article_greek}. To produce yogurt, we first need to prepare standardized milk. We heat the milk to 55 degrees using the heater, then we transfer the heated milk to a homogenizer, which can increase pressure to 10 MPa to break down the yogurt components. After homogenization, we pasteurize the milk at around 90–95°C. Finally, we cool the milk to approximately 40–45°C. The drum-boiler example is based on the standard \texttt{Modelica.Fluid.DrumBoiler} model~\cite{modelica_msl}, as implemented in OpenModelica~\cite{openmodelica_usersguide}. It represents a steam drum boiler with an evaporator, furnace, pump, steam valve, and a PI controller for the drum level. For our verification workflow, we focus on the drum steam temperature $T_S$ and steam pressure $P_S$, and build a state diagram over low/normal/high pressure regions. We then fit a simple two-dimensional linear ODE for $(T_S, P_S)$ from FMU-based simulation traces.

\vspace{-2mm}
\begin{figure}[h]
  \centering
  \includegraphics[width=1\linewidth]{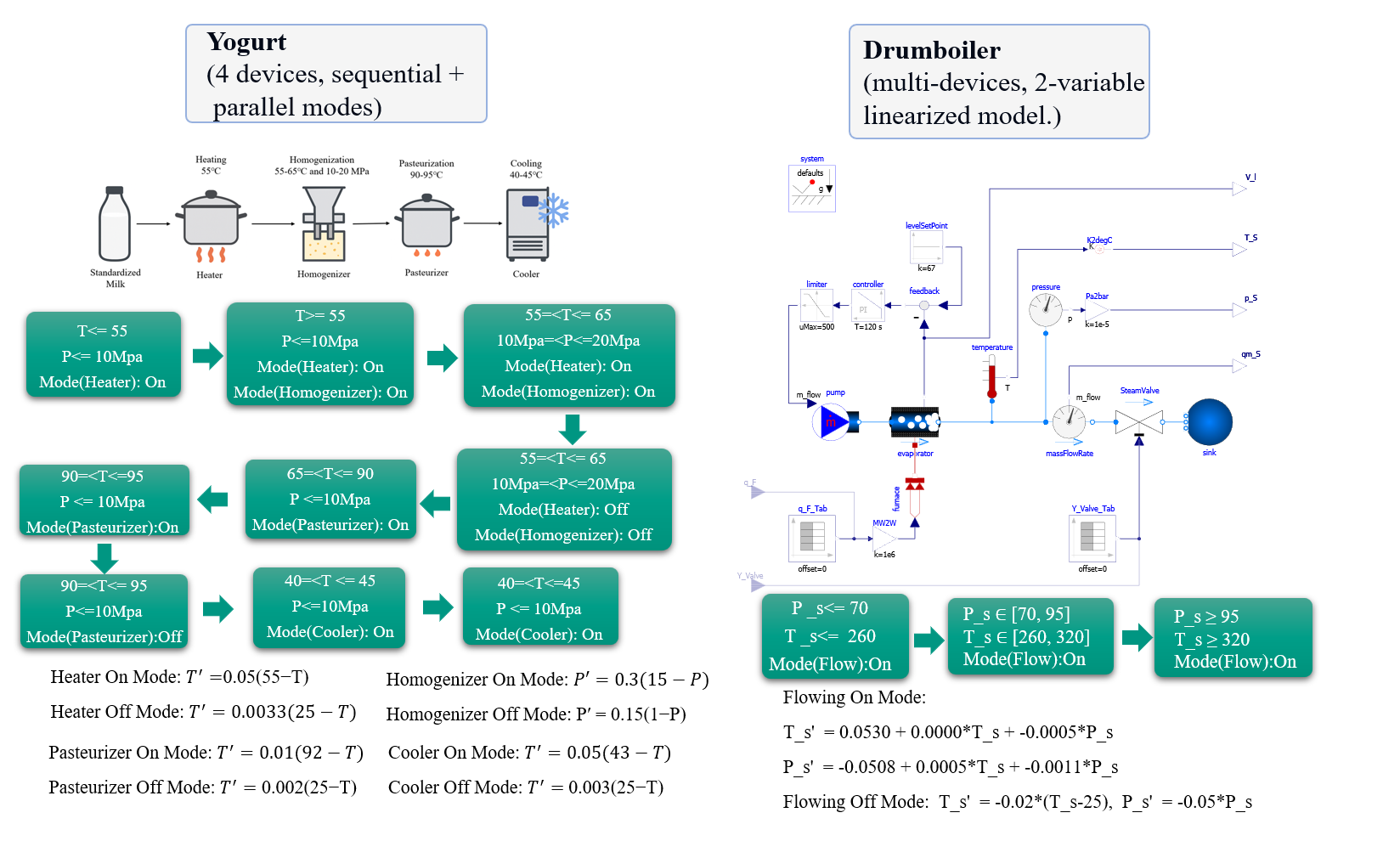}
  \captionsetup{skip=2pt}
  \caption{Yogurt and drum-boiler examples.}
  \label{fig:case_study2}
\end{figure}
\vspace{-2mm}

\begin{figure}[t]
  \centering
  \includegraphics[width=1\linewidth]{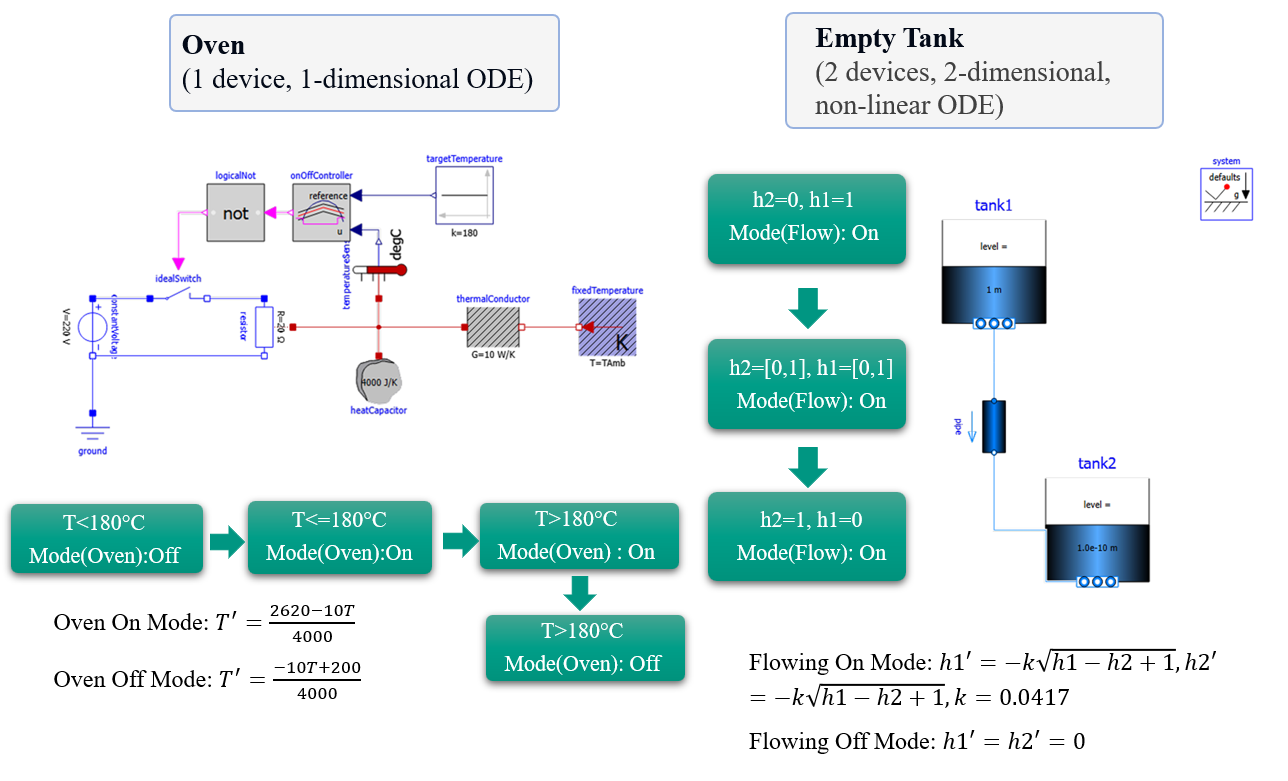}
  \caption{Oven and Tank Example}
  \label{fig:case_study1} 
\end{figure}

\subsection{Formal Graph Invariants}

We evaluate whether the RDF state graph satisfies the transitive closure of its machine-verified state transition relation. The proof of closure, i.e.\ that every composite path permitted by the physical model is represented in the data.

\paragraph{Closure.}
The graph should effectively capture the transitive closure of one-step transitions. If state A can transition to B and B to C (with verified triples for each), then logically A can reach C. We define \(G=(S,E)\), \(\;E=\{\texttt{:next},\;\texttt{:ModeChange}\}\), where $S$ is the finite set of resources of type \texttt{:State} and each edge in $E$ is a machine‑checked transition. For logical completeness we require the transitive closure $E^{*}$ to be contained in the query semantics of the RDF store, \emph{i.e.} if $(s_{i},p,s_{j})\in E$ and $(s_{j},q,s_{k})\in E$ then $(s_{i},E^{*},s_{k})$ is derivable. ASK queries, such as in Listing~\ref{lst:closure-ask}, can check for such closure.

\begin{listing}[h]
\centering
\begin{lstlisting}[%
    language=SPARQL,
    basicstyle=\ttfamily\small,
    numbers=none,
    frame=single,
    breaklines=true
]
ASK WHERE {
  ex:s0 (:next | :ModeChange)+ ex:s4 . }
-- Result: true
\end{lstlisting}
\caption{Closure check: Could \(s_0\) transition to \(s_4\)?}
\label{lst:closure-ask}
\end{listing}

Such queries can return true depending on how the triple store defines its entailment regime\footnote{\url{https://www.w3.org/TR/sparql11-entailment/}}.
Two ways can enable such closure queries using \texttt{+} and \texttt{*} in the query string:
First, to apply the behavior of SPARQL property path operators \texttt{OneOrMorePath} or \texttt{ZeroOrMorePath} in the presence of entailment, as implemented in Apache Jena, where no difference is made between entailed and initially asserted triples when evaluating queries with entailment enabled.
Second, to supply RDF with all entailed triples materialized to triple stores without entailment, or to those that implement the standard behavior (only the Basic Graph Patterns (BGPs) see entailed data during query evaluation).

\subsection{Soundness}

The integrated RDF–dL framework and its properties are described in our reference implementation and case study results. Notably, safety properties verified in the dL domain become available as entailment in the RDF domain for querying. Formal lemmas in our work establish that every verified edge (continuous or discrete) is sound and that any path composed of these edges will preserve global invariants. These guarantees illustrate the logic-based correctness of our approach.

\subsection{Scalability and Applicability}

Table~\ref{tab:results-b} combines \emph{the structural size} (number of states/modes) with the verification dL logic obligations. We observe:

\begin{table}[t]
  \centering
  \small
  \setlength{\tabcolsep}{2pt}
  \renewcommand{\arraystretch}{1.05}
  \begin{tabularx}{\columnwidth}{@{}lccccc>{\centering\arraybackslash}X@{}}
    \toprule
    \textbf{Use Case} &
    \textbf{States} &
    \textbf{Modes} &
    \textbf{\shortstack{State Transitions.}} &
    \textbf{\shortstack{Correct Transitions.}} &
    \textbf{\shortstack{dL Obligations.}} &
    \textbf{\shortstack{Sound\\ / Global\\Inv.}} \\
    \midrule
    Oven       & 4  & 2 & 4  & 4  & 40  & Yes/Yes \\
    Empty Tank & 5  & 2 & 4  & 4  & 36  & Yes/Yes \\
    Yogurt     & 14 & 8 & 18 & 18 & 350 & Yes/Yes \\
    Drumboiler & 5  & 2 & 4  & 4  & 40  & Yes/Yes \\
    \bottomrule
  \end{tabularx}

  \caption{Evaluation Results}
  \label{tab:results-b}
\end{table}

\paragraph{(1) Near‑linear scaling for linear dynamics.}
For use cases with linear or polynomial dynamics (Oven, Drumboiler, and Yogurt process), the number of proof obligations grows approximately
linearly with the number of modes and states.

\paragraph{(2) Non‑polynomial and multi‑dimensional dynamics.} The \textit{Tank} ODE in this example: \(
      h_1'=-k\sqrt{h_1-h_2+1},
      h_2'= k\sqrt{h_1-h_2+1}, \quad k=0.0417.
    \), in \textit{FlowingOn} forces KeYmaera X to switch from polynomial arithmetic to real algebraic certificates, triggering decision procedures with doubly exponential complexity. All proof obligations for the tank system are still discharged successfully, showing that our method is \emph{robust} for two‑dimensional nonlinear flows.

\paragraph{(3) Applicability across different scenarios.}
Our case studies cover:  
(i)~a \emph{single‑device} thermostat loop (Oven),  
(ii)~a \emph{two‑tank} flow system with cross-coupled ODEs (Empty Tank), and
(iii)~a \emph{multi-device} yogurt production line with 14 state changes, and (iv)~an industrial \emph{drum‑boiler} steam generation system. In all cases, the framework automatically
\textsf{(a)} generates the full set of state transitions in Jena,  
\textsf{(b)} discharges all dL obligations in KeYmaera X, and  
\textsf{(c)} proves preservation of the global safety invariant.

\subsection{SHACL shapes for states' validity}
To assess scalability, we repeatedly validate real-time SHACL state validation in three examples (oven, yogurt, and empty tank). We report the shape matches from these runs, shown as Table~\ref{tab:shacl-correct-shape}.

\begin{table}[h]
  \centering
  \caption{Valid real-time state mapping to correct SHACL shapes.}
  \label{tab:shacl-correct-shape}
  \scriptsize
  \begin{tabular}{p{1.1cm}p{4.5cm}p{1cm}}
    \hline
    Domain & Real-time State & Matched State\\
    \hline
    Oven   & Oven: On, x=60 & s12\\
    Oven   & Oven: Off, x=150 & s11\\
    Oven   & Oven: On, x=190 & s21\\
    Oven   & Oven: Off, x=200 & s22\\
    Yogurt & Heater: On, x=45; Homogenizer: Off, p=9 & s112\\
    Yogurt & Heater: Off, x=50;  Homogenizer: Off, p=8 & s122\\
    Yogurt & Heater: On, x=60; Homogenizer: On, p=15 & s311\\
    Yogurt & Heater: On, x=110; Homogenizer: Off, p=2 & s212\\
    Tank   & Tank1: Off,h1=1.0 ; Tank2: Off,h2=0.1 & s0\\
    Tank   & Tank1: On,h1=1.0 ; Tank2: On,h2=0.1 & s1\\
    Tank   & Tank1: Off,h1=0.5 ; Tank2: Off,h2=0.5 & s2\\
    Tank   & Tank1: On,h1=0.0 ; Tank2: On,h2=1.0 & s3\\
    Tank   & Tank1: Off,h1=0.0 ; Tank2: Off,h2=1.0 & s4\\
    \hline
  \end{tabular}
\end{table}

Even states that are not covered by the example instances are still detectable: out-of-range readings or inconsistent mode combinations are immediately flagged as invalid by SHACL. We omit those invalid rows here to keep the table focused on the correct state correspondences. The reason for using SHACL to validate real-time states is to provide the real-time physical system states to the RDFdL: only an observed state in the RDFdL can be used for state transitions and next-state reasoning. In other words, SHACL filters invalid real-time states, and RDFdL then reasons on the correct state to check which are the following state transitions. This design allows continuous runtime validation and transition verification in the same pipeline.

\section{Related Work}
\label{sec:Related_work}

\textbf{Semantic‐web AAS models.} Industry 4.0 and the rise of Smart Factories have brought an increasing focus to digital twins, virtual representations of physical industrial assets that remain linked to their real counterparts \cite{ARM2024275}. Asset Administration Shell (AAS) is an example implemented as a standardized container for an asset’s digital twin information in Industry 4.0 frameworks \cite{ABDELATY20222533}. The AAS serves as an industrial resource's "virtual envelope, organizing all relevant data about the asset throughout its life cycle. The AAS is often enriched with Semantic Web technologies: formal ontologies defined in RDF/OWL represent AAS data models \cite{9468266}. Recent research recognizes that AAS and semantic RDF models have strengths that complement each other \cite{RONGEN2023103910}—integrating knowledge graphs with Industry 4.0 digital twin frameworks to enable interoperability and machine-readable intelligence.

Semantic web technologies have been applied to model industrial control systems and manufacturing processes, capturing their static structure and ensuring data consistency. Ontologies expressed in RDF/OWL can represent triples in a production system, enabling the integration of heterogeneous engineering data. For example, \cite{Feldmann2016} illustrates how combining RDF graphs with SPARQL queries can detect inconsistencies between interdisciplinary engineering models in an automated production system. SHACL has been used to enforce domain constraints, e.\,g.\ \cite{9779174} uses an ontology and SHACL to verify that an industrial control system design meets the IEC-62443 security requirements. These semantic models and constraints provide a machine-readable, declarative specification of the industrial process but remain static. In other words, while RDF/SHACL can validate structural correctness or configuration compliance, they do not capture the dynamic behavior or control logic of the physical processes. The limitation is that semantic models stop at describing what the system is, without analyzing what the system does dynamically.

\textbf{Formal verification of CPS.} A Cyber-Physical System (CPS) is one of the main concepts in Industry 4.0, providing interaction between the physical and virtual worlds \cite{LINS2020106193}. Formal verification of CPS addresses the challenge of proving safety and correctness in systems that mix control logic with continuous physical processes \cite{Platzer18}. In dL, one can write specifications such that if certain initial conditions hold, a safety property remains true after any control decisions and continuous evolutions. The dL theorem prover KeYmaera X is used to verify CPS. KeYmaera X can prove invariants and temporal properties by reasoning symbolically about differential equations and control choices \cite{10.1007/978-3-319-21401-6_36}. For example, dL can prove that a robot arm will never collide with a wall or a chemical reactor will never overheat under its control program \cite{Platzer18}. These verifications are machine-checked proofs that offer a high level of assurance. 

\textbf{Discrete dynamics and the Semantic Web}. Few works consider ontological modelling with the modelling of the dynamics of systems.
Notably, some work exists that models processes as discrete steps, e.\,g.
\cite{10.1007/978-3-030-54994-7_9} uses an ontology-based knowledge base to support verifying requirements against a description of a concurrent procedural flow in manufacturing, where the dynamics are represented as discrete steps. \cite{DBLP:conf/semweb/KaferH18} takes inspiration from business process modelling to define a process ontology, using which processes can be modelled that are executable while honouring the assumptions of semantic web technologies. \cite{DBLP:journals/jair/HaririCMGMF13} Investigate discretely evolving description logic knowledge bases. Those works, however, cannot consider continuous descriptions given in ODEs.

\section{Discussion}
\label{sec:discussion}

\textbf{Why RDFdL?} Using SPARQL, we can query RDF graphs, but it has no semantics for continuous dynamics. Pure dL proves reachability over ODE‐based modes, yet knows nothing about metadata or master data. The running example from Section~\ref{sec:Introduction} illustrates this gap: we want to talk about \emph{buttons} and \emph{technicians} while also requiring reachability along a dL‑verified path. RDFdL closes this gap: every dL‐verified transition materialized as an RDF triple that can be queried using the predicates  
\(
  \texttt{:next}\;\bigl|\;\texttt{:ModeChange}
\).
Hence, a SPARQL query can perform a reachability query over the state space while, at the same time, following static RDF links  
\(
  \textit{button}\;\to\;\textit{device}\;\to\;\textit{technician}.
\),
see Listing~\ref{lst:sparql-tech}.
This demonstrates that RDFdL combines the rich metadata of RDF and SPARQL queries with proven dL lemmas for continuous behavior.

Beyond the running example, RDFdL also supports reachability-style queries that extract complete “happy paths’’ of verified transitions between two states. Appendix~\ref{sec:application} illustrates this for the oven case.

\section{Conclusion}
\label{sec:Conclusion}

We introduced RDFdL, a framework to link knowledge graphs (RDF) with Differential Dynamic Logic (dL). We presented (i) an RDF-based representation of hybrid systems that captures continuous dynamics via ODEs and SHACL-based state definitions, (ii) an automated pipeline that extracts dL proof obligations from RDF and verifies them with the KeYmaera X prover, and (iii) the RDFdL framework that exposes verified transitions as triples for SPARQL engines. We showed that RDFdL scales to multiple industrial-style case studies and that the number of proof obligations grows roughly linearly with the number of modes while still handling nonlinear, multidimensional ODEs.

\section*{Supplemental Material Statement}

Supplemental codes and simulation models for all the case studies we introduced in the paper are available at \url{https://anonymous.4open.science/r/RDFdL-2026}. The repository contains the full implementation of the RDFdL pipeline and build scripts that are needed to reproduce the results in the paper. And the ontology information is listed in \url{https://purl.archive.org/rdfdl/vocab}. 

\section*{Appendix}

\section{ODE}
\label{ssec:ODE}

In the oven example. We introduce individuals or blank nodes that represent state variables and the ODE. For example, \texttt{ex:OvenOnMode} and \texttt{ex:OvenOffMode} have property \texttt{ex:hasDevice} which connects the \texttt{ex:Oven} (a \texttt{Heater} device). And they may have a property \texttt{ex:hasODE} linking it to a blank node that encodes the oven's ODE. RDF details are shown in the listing \ref{lst:rdf-oven-onmode}.

\begin{listing}[t]
  \scriptsize
  \captionsetup{skip=2pt}
  \caption{RDF encoding of \texttt{OvenOnMode}.}
  \label{lst:rdf-oven-onmode}
  \begin{lstlisting}[frame=single, basicstyle=\ttfamily\scriptsize,
                     numbers=none, breaklines=true]
ex:OvenOnMode rdf:type     ex:ModeRecord ;
    ex:hasDevice           ex:Oven ;
    ex:hasMode             ex:OnMode ;
    ex:hasODE [ rdf:type ex:ODE ;
                ex:derivative             "0.605 - 0.0025*(x - 20)" ;
                ex:startingCondition      "x <= 180" ;
                ex:evolutionDomainConstraint "x <= 200" ;
                ex:evolvingVariable       ex:x
              ] .
  \end{lstlisting}
\end{listing}

\section{State representation in SHACL}
\label{ssec:state_in_Shacl}
Below, we provide the complete definitions of the SHACL node shapes and the states. NodeShapes are shown here. Listing~\ref{lst:s11-shapes} shows the SHACL shapes defining the oven on mode and its state:

\begin{listing}[h]
  \scriptsize
  \captionsetup{skip=2pt}
  \caption{SHACL shape for \texttt{ex:s11} (oven example).}
  \label{lst:s11-shapes}
  \begin{lstlisting}[frame=single, basicstyle=\ttfamily\scriptsize,
                     numbers=none, breaklines=true]
ex:s11OvenShape rdf:type sh:NodeShape ;
  sh:property [ sh:path ex:x ; sh:minCount 1 ; sh:maxCount 1 ;
                sh:maxInclusive "180"^^xsd:long ] ;
  sh:property [ sh:path ex:mode ; sh:hasValue ex:OffMode ;
                sh:minCount 1 ; sh:maxCount 1 ] .

ex:s11Shape rdf:type sh:NodeShape ;
  sh:property [ sh:path ex:hasOven ; sh:node ex:s11OvenShape ] ;
  sh:targetClass ex:State .

ex:s11 rdf:type ex:State ;
  ex:hasOven ex:Oven ;
  ex:hasShape ex:s11Shape .
  \end{lstlisting}
\end{listing}

\section{More Application}
\label{sec:application}

The engineer often specifies a “happy path,” i.e., an idealized sequence of states from the initial state to the goal state under nominal conditions. For the Oven example, we run the SPARQL query in Listing~\ref {lst:oven-happy-path} to extract all verified transitions on the "happy path" from the starting state \texttt{ex:s11} to \texttt{ex:s22} ending state. The table shows the results of executing this SPARQL query, which align perfectly with the ideal workflow described in the use case; every state transition matches the process steps in Section \ref{subsec:case_study}. 

\begin{listing}[h]
  \scriptsize
  \captionsetup{skip=2pt}
  \caption{SPARQL query and extracted “happy-path” transitions (oven).}
  \label{lst:oven-happy-path}
  \begin{lstlisting}[language=SPARQL,frame=single, basicstyle=\ttfamily\scriptsize,
                     numbers=none, breaklines=true, breakatwhitespace=true]
SELECT ?fromState ?toState ?predicate ?predLabel WHERE {
  ?fromState ?predicate ?toState .
  ex:s11 (:next|:ModeChange)* ?fromState .
  ?toState (:next|:ModeChange)* ex:s22 .
  OPTIONAL { ?predicate rdfs:label ?predLabel }
} ORDER BY ?fromState
  \end{lstlisting}
  \vspace{1mm}
  \begin{tabularx}{\linewidth}{@{} l l X l @{}}
    \toprule
    fromState & predicate & predLabel & toState \\
    \midrule
    \texttt{ex:s11} & \texttt{ex:Oven\_OffMode\_to\_OnMode} &
      "Change Off → On" & \texttt{ex:s12} \\[2pt]
    \texttt{ex:s12} & \texttt{:next} & & \texttt{ex:s21} \\[2pt]
    \texttt{ex:s21} & \texttt{ex:Oven\_OnMode\_to\_OffMode} &
      "Change On → Off" & \texttt{ex:s22} \\[2pt]
    \texttt{ex:s22} & \texttt{:next} & & \texttt{ex:s11} \\
    \bottomrule
  \end{tabularx}
\end{listing}

\appendix

\bibliography{tgdk-v2021-sample-article}

\end{document}